%% file: iccc.tex
\documentclass[letterpaper]{article}
\usepackage{iccc}
\usepackage{amsmath}
\usepackage{amsthm}
\usepackage{amsfonts}
\usepackage{times}
\usepackage{helvet}
\usepackage{booktabs}
\usepackage{array}
\usepackage{caption} 
\usepackage{ragged2e} 
\usepackage{tikz}
\usepackage{courier}
\usepackage{centernot}
\usepackage{tikz,adjustbox}
\usetikzlibrary{arrows.meta,positioning,matrix}
\input{notation}

\title{Transformational Creativity in Science: A Graphical Theory}
\author{Samuel Schapiro\\ 
University of Illinois Urbana-Champaign \\ 
Spiral Works \\
sjs17@illinois.edu
\And Jonah Black  \\
University of Illinois Urbana-Champaign \\
jblac8@illinois.edu
\AND
Lav R.\ Varshney \\ 
 University of Illinois Urbana-Champaign\\
varshney@illinois.edu
}
\begin{document} 
\maketitle
\begin{abstract}
\begin{quote}
\input{sections/abstract}

\end{quote}
\end{abstract}

\section{Introduction}
\input{sections/introduction}
\input{sections/background}

\section{Characterizing Transformational Scientific Creativity Using Graphs}
\input{sections/framework_new}

\section{Towards AI Super-Scientists: Using Our Framework for Transformative Discovery}
\input{sections/ai_superscientists}

\section{Conclusion}
\input{sections/conclusion}

\section{Acknowledgments}
\input{sections/acknowledgements}

\bibliographystyle{iccc}
\bibliography{iccc}

\end{document}

%% file: notation.tex
\newcommand{\beq}{\begin{equation}}
\newcommand{\eeq}{\end{equation}}

\newcommand {\commentout}[1] {}

\def\ints{{{\rm Z} \kern -.35em {\rm Z} }}  % ints
\def\smallints{{{\rm Z} \kern -.3em {\rm Z} }}  % small ints
\def\pints{{{\rm I} \kern -.15em {\rm N} }}      % pints
\newcommand{\reals}{\mathbb R}

\def\cplx{{{\rm I} \kern -.45em {\rm C} }}       % complex
\def\l2{\rm {\mathcal L}^{2}(\reals)}            % l2

\newcounter{globalthm}

\newtheorem{theorem}[globalthm]{Theorem}
\newtheorem{definition}[globalthm]{Definition}

\newtheorem{nad}[globalthm]{Notation and Definitions}

\newcommand{\be}{\begin{eqnarray}}
\newcommand{\ee}{\end{eqnarray}}
\newcommand{\bea}{\begin{eqnarray}}
\newcommand{\eea}{\end{eqnarray}}
\newcommand{\beaa}{\begin{eqnarray*}}
\newcommand{\eeaa}{\end{eqnarray*}}
\newcommand{\bnad}{\begin{nad}}
\newcommand{\enad}{\end{nad}}

\newcommand{\calA}{{\cal A}}

%% file: sections/abstract.tex
Creative processes are typically divided into three types: combinatorial, exploratory, and transformational. Here, we provide a graphical theory of transformational scientific creativity, synthesizing Boden's insight that transformational creativity arises from changes in the ``enabling constraints" of a conceptual space \cite{boden92} and Kuhn's structure of scientific revolutions as resulting from paradigm shifts \cite{kuhn1997structure}. We prove that modifications made to axioms of our graphical model have the most transformative potential and then illustrate how several historical instances of transformational creativity can be captured by our framework.

%% file: sections/introduction.tex
\citeauthor{boden92} \shortcite{boden92} famously divides creative processes into three types: combinatorial, exploratory, and transformational creativity. Combinatorial creativity refers to the combination of concepts and ideas; exploratory creativity involves the traversal of a conceptual space; and transformational creativity entails restructuring the space itself. Historically, groundbreaking scientific discoveries—including Einstein's theory of relativity, the Copernican shift from geocentrism to heliocentrism, and the discovery of non-Euclidean geometry—have fallen under the category of transformational creativity. However, despite scientific creativity being a heavily studied topic in the philosophy and psychology of creativity~\cite{creativity_in_science,flight_from_wonder,simonton2021scientific,koestler_creation}, rigorous formalizations of transformational creativity in science have remained underexplored. 

In this short paper, we formalize scientific conceptual spaces as directed acyclic graphs (DAGs), with vertices that function as generative rules for artifacts and edges that represent logical dependencies among axioms and rules. From our graphical definition, we prove that modifications made to axioms of the graph have the most transformative potential. Finally, we explain how our model can capture historical instances of transformational scientific creativity and discuss potential extensions of this paper in future work. To start, we provide a background on transformational creativity and an account of its role in scientific revolutions.

%% file: sections/background.tex
\subsection{Background on Transformational Creativity}
The notion of transformational creativity is first introduced by \citeauthor{boden92} \shortcite{boden92} as a distinct form of the creative process. Unlike combinatorial creativity that involves novel combinations of familiar elements, or exploratory creativity that entails exploring a defined conceptual space, transformations specifically target what Boden calls the ``enabling constraints" of a conceptual domain: the foundational assumptions, rules, or boundaries that determine what artifacts are intelligible within it \cite{boden92,aguiar2015semiotic}. When these constraints are modified, the generative landscape is restructured, and the scope of new artifacts becomes drastically altered.

Recent work has attempted to give computational grounding to this notion. \citeauthor{santo2024formalcreativitytheorypreliminary} \shortcite{santo2024formalcreativitytheorypreliminary} distinguish between novelty and transformation using concepts from formal learning theory (FLT), showing how transformational creativity can be operationalized as a shift in the hypothesis space informed by new data. Their framework centers the agent’s capacity to adjust the criteria governing the conceptual space, rather than operating under a fixed set of rules. Similarly, \citeauthor{toivanen2018novelsongs} \shortcite{toivanen2018novelsongs} describe an architecture for transformational creativity in generative systems, where meta-level modules modify the constraints that govern the search space of songs to produce novel outputs. This coheres with previous work by \citeauthor{wiggins2006} \shortcite{wiggins2006}, who defines transformational creativity as exploratory search over the space of generative systems that collectively define a conceptual space. \citeauthor{ritchie2006transformational} \shortcite{ritchie2006transformational} takes a similar view, emphasizing the idea that transformational creativity requires modifications to the conceptual structure of a space.

\subsection{Transformational Creativity in Science: Puzzle Solving, Crisis, and Revolution}
Kuhn's history and philosophy of science articulates a paradigmatic view in which periods of normal science---consisting of problem solving within a predefined conceptual space---are followed by anomalies that underscore the weaknesses and inconsistencies of the existing space. These crises prompt scientists to return to first principles and modify existing paradigms or develop new ones, culminating in a scientific revolution that alters a previous worldview \cite{kuhn1997structure}. This model of the scientific process is largely consistent with the three forms of creativity (combinatorial, exploratory, and transformational) popular in the computational creativity community \cite{boden92,varshney_limit_theorems,ritchie2006transformational}. Below, we summarize Kuhn's model, which is later adopted as a conceptual guide for parts of our theory. 

\subsubsection{Phase I: Normal Science} Normal science consists of research based firmly on past scientific achievements that create the foundation for further scientific practice. These achievements establish a conceptual space (or a ``paradigm") for organizing otherwise disconnected facts while also motivating the discovery of new ones. The paradigm also offers practical advantages for new work: because scientists do not need to re-justify widely accepted concepts, they can focus on core contributions that often entail exploratory or combinatorial creativity. Kuhn ascribes the term ``puzzle solving" to the period of normal science and highlights that the goal of this period is \textit{not} to produce major novelties, but rather to extend the existing paradigm, often only incrementally.

\subsubsection{Phase II: Anomaly and Extraordinary Science} The limitations of working within pre-defined paradigms are often highlighted by periods of anomaly and crisis, when the ``puzzle solving" phase, characterized by exploration of existing conceptual spaces, is no longer sufficient to account for the range of phenomena the field wishes to study. This marks the beginning of an \emph{extraordinary science} phase, hallmarked by a loosening of existing rules that promotes the eventual emergence of new paradigms. Extraordinary science often involves a return to first principles, which may entail a ``recourse to philosophy" and ``debate over fundamentals". 

\subsubsection{Phase III: Scientific Revolutions}
The resolution of a scientific crisis is a scientific revolution, including a paradigm shift and a change of worldview. This phase involves a reconfiguration of the constraints that previously governed all inquiry within a scientific domain. As a result, new artifacts abound, and their implications propagate throughout the space. In Kuhn's words:

\begin{quote}
It's as if the scientific community has been transported to another planet, where familiar objects appear in a new light and are accompanied by previously unseen ones. \cite[p.~111]{kuhn1997structure}
\end{quote}

\noindent In the following section, we synthesize Kuhn's structure of scientific revolutions with prior formalizations of transformational creativity into a graphical model.

%% file: sections/framework_new.tex
\input{sections/graph}
In this section, we use a formal, graphical structure to define conceptual spaces, clarify the distinction between axioms and rules, and illustrate how existing constraints influence the discovery and articulation of new scientific artifacts. In the science of science, it is often common to use graphical or network frameworks to model the act of, or phenomena related to, scientific discovery \cite{shietalweavingfabricofscience,evanssuprisingcombinations2023,wiggins2006,friedman2000directedcausallearning}. Likewise, our approach focuses on dependency relationships between scientific rules and axioms, allowing us to characterize how changes to axioms may potentially transform conceptual spaces. In contrast to prior, more general accounts of transformational creativity \cite{boden92,ritchie2006transformational,santo2024formalcreativitytheorypreliminary}, we focus specifically on this phenomenon in science.

\subsection{Scientific Conceptual Spaces and Artifacts}
In science, an artifact can be any phenomenon, fact, concept, variable, constant, technique, theory, law, question, goal, or criterion of relevance to a particular field \cite{creativity_in_science}. As explained in \citeauthor{boden92} \shortcite{boden92},  conceptual spaces are defined in terms of the ``enabling constraints" that assert criteria for acceptable and interpretable artifacts. To formalize this notion, we start by adopting similar notation as \citeauthor{santo2024formalcreativitytheorypreliminary}  \shortcite{santo2024formalcreativitytheorypreliminary}, described below:
\begin{enumerate}
    \item \textbf{A formal language:} $L \in \mathcal{E}$, where $\mathcal{E}$ represents a class of recursively enumerable sets. This defines a space of possible artifacts.
    \item \textbf{Constraint:} Any subset $L’ \subseteq L$. This constrains the space of possible artifacts.
    \item \textbf{A potential artifact:} Any string $w \in L$. Later, in Definition~\ref{def:artifact}, we define what it means for an artifact to obey the constraints of a scientific conceptual space. 
\end{enumerate}

\noindent Now, we introduce our graphical definition of a scientific conceptual space.

\begin{definition}[Scientific Conceptual Space]\label{eq:concept_space}
A conceptual space $S$ is a finite, directed acyclic graph (DAG)
\[
    S \;=\; G = (V,E), \qquad V \subseteq \mathcal{P}(L),\; E \subseteq V \times V,
\]
where each vertex $v \in V$ represents a subset $v \subseteq L$. The DAG represents a dependency structure over constraints, with the following properties:

\begin{enumerate}
    \item \textbf{Axioms.}  
          The set of sink nodes
          \[
              \mathcal A \;=\;
              \bigl\{A \in V \mid \not \exists\,u \in V:\,(A, u) \in E \bigr\}
          \]
          are the axioms of $S$ and have no outgoing edges.
    \item \textbf{Edge dependencies.} $(u,v) \in E$ iff $u \subseteq v$, meaning $u$ is a further constraint on $v$, so that $v$ is a necessary condition for $u$. By construction, each $a \in \cal{A} $ satisfies the property of non-dependence:
          $$ \centernot \exists v \in V \; (A \subseteq v)\mbox{.}$$
    \item \textbf{Rules.} A rule is any vertex $R \in V \setminus \calA$. A rule is distinguished from an axiom in that it depends on other nodes in the graph, while axioms are self-justifying and define the initial scope of artifacts worth considering.
    \item \textbf{Rule Prerequisites.}  
          Given any node $v \in V$, the set of constraints on which $v$ depends is given by:
          $$
          \text{prereq}(v) := \{ u \in V \mid \text{there is a path from $u$ to $v$ on $S$} \}  \mbox{.}
          $$
          This represents all the nodes that must be assumed before the constraint $v$ can be invoked at all.
    \item \textbf{Rule Dependents.}
         Given any node $v \in V$, the set of constraints that depend on $v$ is given by
         $$
         \text{depends}(v) := \{ u \in V  \mid \text{there is a path from $u$ to $v$ on $S$}\} \mbox{.}
         $$
         This represents the set of all other constraints that cannot be invoked without first assuming $v$.
\end{enumerate}
\end{definition}

\input{sections/table}

In \textit{Phase I: Normal Science} of Kuhn's model, a paradigm's structure and constraints are assumed as a foundation to explore within the conceptual space they induce \cite{kuhn1997structure}. Likewise, in our definition, axioms do not depend on other constraints but are instead assumed for the possibility of inquiry within the space. Also note that no connectedness assumption about the graph is made, meaning that the graph can have multiple connected components to reflect disconnected but related subgraphs of knowledge, corresponding to subsets of constraints over possible artifacts. A new artifact is generated from a conceptual space in the following way.

\begin{definition}[Scientific Artifact]\label{def:artifact}
Given a conceptual space $S = (V, E)$, a scientific artifact is an ordered triple
\[
      a \;=\; (h,\,\sigma, w), \qquad h \in \mathcal H,\; \sigma \subseteq V,\;w \in \sigma,\
\]
where
\begin{enumerate}
  \item $h$ is a header drawn from a domain-specific set $\mathcal H$ (e.g.\ ``phenomenon'', ``law'', ``measuring‐technique'', ``open‐question'') that fixes the kind of the artifact.
  \item $\sigma := \{V_1, \dots, V_n\}$ is a non-empty support set of vertices in $S=(V,E)$, representing a subset of constraints whose joint satisfaction makes the artifact well-defined or intelligible.
  \item $w \in \bigcap_{i=1}^n V_i$ is a string representing the artifact.
\end{enumerate}
\end{definition}

\subsection{Transformational Scientific Creativity}
Transformational creativity occurs when modifications are made to a conceptual space. Our definition of a space $S = (V, E)$ naturally supports transformations in two ways.
\begin{enumerate}
    \item \textbf{Edge Modifications ($\oplus, \ominus$):} An edge $e \in E$ is removed ($\ominus$) and/or a new edge $e' \notin E$ is added ($\oplus$).
    \item \textbf{Vertex Modifications ($\oplus, \ominus, \prime$):} A vertex $v \in V$ is removed ($\ominus$) and/or a new vertex $v' \in \cal{A} \setminus V$ is added ($\oplus$). This also encompasses the case in which the content of a vertex is modified $(\prime)$.
\end{enumerate}

For simplicity, we consider the case of modifying an existing node rather than adding a new one or removing an existing one. Still, in future work, our results can be extended to these cases, as well as to the modification of edges. Below, we quantify the transformative potential of nodes in the conceptual space's graph.

\begin{definition}[Transformativeness via Modification] \label{def:trans_modification}
Given a node $v \in V$, its transformative potential is the number of nodes that depend on $v$, given by 
$$T^p_\text{mod}(v) := | \text{depends}(v)|$$
Since it is not possible to invoke any rule in $\text{depends}(v)$ without assuming or observing $v$, if $v$ is modified, then this has the potential to transform the structure of every node in $\text{depends}(v)$.
\end{definition}

As articulated in \textit{Phase III: Scientific Revolutions} of \citeauthor{kuhn1997structure} \shortcite{kuhn1997structure}, revolutions occur when the foundational principles of existing spaces are modified, leading to paradigm shifts. Our definition of a scientific conceptual space naturally has the property that modifying its most foundational principles (axioms) has the most transformative potential. We prove that this is the case below.

\begin{theorem}[Modifying Axioms has the Most Transformative Potential] \label{thm:axioms}
    Given a conceptual space $S$ for which $V \setminus \calA \neq \emptyset$, the node with the greatest transformative potential when modified $T^p_\text{mod}$ must be an axiom. Formally stated, $$\text{argmax}_{v \in V} T^p_\text{mod}(v) \in \calA$$
\end{theorem}

\begin{proof}
Let $R$ be the non-axiom node in $V \setminus \mathcal{A}$ with the largest number of dependents (transformative potential):
\[
R := \underset{v \in V \setminus \mathcal{A}}{\arg\max} ~ |\mathrm{depends}(v)| = \underset{v \in V \setminus \mathcal{A}}{\arg\max} ~ |T^p_\text{mod}(v)| 
\]
Assume for contradiction that $R$ also has the most transformative potential among all nodes in the graph, including axioms.

Since $R$ is not an axiom, there exists some $u \in V$ such that $(R, u) \in E$, meaning that $R$ depends on $u$. By the structure of DAGs, any dependent of $R$ is also a dependent of $u$. Therefore, $u$ satisfies 
\[
|\mathrm{depends}(u)| \geq |\mathrm{depends}(R)| + 1
\]
meaning that it has strictly more dependents than $R$. However, since we defined $R = \arg\max_{v \in V \setminus \mathcal{A}} ~ |\mathrm{depends}(v)|$, $u$ cannot be in $V \setminus \mathcal{A}$ as well, so we must have that $u \in \calA$. Therefore, we have found an axiom with more dependents (transformative potential) than $R$, leading to a contradiction. Since $R$ was chosen arbitrarily, this means there is no rule $v \in V \setminus \calA$ with more transformative potential than the axiom with the most transformative potential, which is what we wanted to show.
\end{proof}

\subsection{Conceptual Space Operations}
\citeauthor{ritchie2006transformational} \shortcite{ritchie2006transformational} presents five critical operations that conceptual spaces should permit. Below, we explain how our definition supports each.
\begin{enumerate}
    \item \textbf{Locating an artifact $A$ within $S$:} An artifact $a = (h, \sigma, w)$ can be located according to the necessary set of vertices $\sigma$ given in Definition~\ref{def:artifact} which make $a$ intelligible.
    \item \textbf{Rate artifacts $A, B$ for similarity with respect to $S$:} Given artifacts $a, b$, their similarity with respect to $S$ can be computed using a Jaccard similarity over their vertex support sets \cite{jaccard1901etude}.
    \item \textbf{Induce a space definition from artifacts $A_1, \dots, A_n$:} Given artifacts $a_1, \dots, a_n$ and the corresponding set of vertices $V_a := \bigcup_{i=1}^n \sigma_i$ which make the artifacts well-defined or intelligible, one can construct a DAG by learning or discovering the edge dependencies between the set of vertices $V_a$. 
    \item \textbf{Given existing artifacts, generate a new artifact:} After arranging existing artifacts into a conceptual space $S$, a new artifact can be generated according to Definition~\ref{def:artifact}.
    \item \textbf{Create a revised definition of $S \to S'$ to include artifact $A$}: This can be done by applying vertex or edge modifications to the graph.
\end{enumerate}

\subsection{Historical Illustrations} 
Our definition of a conceptual space supports interpretations of historical instances of transformational scientific creativity. Importantly, in each case, our model does not attempt to reconstruct the steps used by the scientist(s) making their transformation, but instead attempts to provide an explanation for why such instances were so transformative. In Table~\ref{tab:transformational_creativity}, we provide three examples of transformational creativity, explained in more detail below.

\subsubsection{From Geocentrism to Heliocentrism}
In the shift from Ptolemaic astronomy to a heliocentric system, the axiom that the Earth was the center of the universe $(A_1)$ was modified to reflect a heliocentric model with the sun at the center $(A_1^\prime)$. Rules in $S$ which previously depended on $A_1$ necessarily had to be transformed after the axiom modification $A_1^\prime$ was made.

\subsubsection{Relativity Theory} In the shift from Newtonian ($S$) to relativistic mechanics ($S^\prime$), space and time were originally considered absolute (axiom $A_1$), so that momentum obeyed the formula $p = mv$ (rule $R_1$), but relativity theory revealed that this original formulation only held locally for $v \ll c$. When the new axioms that physical laws are Lorentz invariant (axiom $A_1^\prime$) and the speed of light is constant (axiom $A_2^\prime$) 
were introduced, the momentum equation had to be reformulated to reflect the dependence on these new underlying principles, yielding rule $R_1^\prime$: that $p = \frac{mv}{\sqrt{1 - v^2/c^2}}$ 

\subsubsection{Non-Euclidean Geometry} The discovery of non-Euclidean geometry, which removed the parallel postulate $P$ from the prior Euclidean axiom set $\calA$, yielded a new axiom set $\calA' := \calA \setminus P$. As a result of this transformation, derived rules such as the angle sum of a triangle equaling 180 degrees ($R_1$) were modified to reflect that the computation of angular sums no longer depended on the parallel postulate ($R_1'$).

%% file: sections/graph.tex
\begin{figure*}[t]
\centering
% ------------------------------------------------------------------
%  Shared TikZ settings
\usetikzlibrary{arrows.meta,positioning}
\tikzset{
  node/.style={
    draw, rounded corners, align=center, font=\scriptsize,
    minimum width=27mm, inner sep=2pt
  }
}
%                                                                  %
% =====================  (a)  GEOCENTRISM  ======================== %
%                 %
\begin{minipage}[t]{0.48\textwidth}
\centering
\scalebox{0.82}{%
\begin{tikzpicture}[>=Stealth, node distance = 1.25cm and 1.15cm]
  % -------- level 0 --------
  \node[node] (PP){Planetary\\Predictions\\\textit{Almagest} ($V_1$)};
  % -------- level 1 --------
  \node[node,below=of PP] (ED){Epicycle--\\Deferent Model ($V_2$)};
  % -------- axioms / sinks --------
  \node[node,below left=of ED,yshift=-2mm] (G1)
       {\textbf{Axiom $A_1$}\\Earth Stationary\\\& Central };
  \node[node,below right=of ED,yshift=-2mm] (G2)
       {\textbf{Axiom $A_2$}\\Uniform Circular\\Motion };
  % -------- edges --------
  \draw[->] (PP) -- (ED);
  \draw[->] (ED) -- (G1);
  \draw[->] (ED) -- (G2);
\end{tikzpicture}}
\caption*{(a)\;Geocentric conceptual space $S$}
\end{minipage}
\hfill
%                                                                  %
% =====================  (b)  HELIOCENTRISM  ====================== %
%                                                                  %
\begin{minipage}[t]{0.48\textwidth}
\centering
\scalebox{0.82}{%
\begin{tikzpicture}[>=Stealth, node distance = 1.25cm and 1.15cm]
  % -------- level 0 --------
  \node[node] (PP2){Planetary\\Predictions ($V_1^{\prime}$)};
  % -------- level 1 --------
  \node[node,below=of PP2] (HM){Heliocentric\\Orbital Model ($V_2^{\prime}$)};
  % -------- axioms / sinks --------
  \node[node,below left=of HM,yshift=-2mm] (H1)
       {\textbf{Axiom $A_1^{\prime}$}\\Sun Central};
  \node[node,below right=of HM,yshift=-2mm] (H2)
       {\textbf{Axiom $A_2^{\prime}$}\\Elliptical Orbits };
  % (uniform motion is unchanged, so no prime)
  \node[node,right=of PP2,xshift=3mm] (H3)
       {Earth Rotation\\24‑h Day ($V_3^{\prime}$)};
  % -------- edges --------
  \draw[->] (PP2) -- (HM);
  \draw[->] (HM)  -- (H1);
  \draw[->] (HM)  -- (H2);
  \draw[->] (H3) -- (HM);
\end{tikzpicture}}
\caption*{(b)\;Heliocentric conceptual space $S'$}
\end{minipage}
% ------------------------------------------------------------------
\caption{Conceptual‑space DAGs before ($S$) and after ($S'$)
         the Copernican and Keplerian revolution.   All
         non‑axiom artifacts reach at least one axiom (sink node),
         satisfying the formal definition of a conceptual space.}
\label{fig:geo_vs_helio_dags}
\end{figure*}

%% file: sections/table.tex
\begin{table*}[tb]
    \centering
    \caption{Historical instances of transformational creativity in science, demonstrating how the transformation of axioms $A \to A'$ transformed the conceptual spaces $S \to S'$.}
    \renewcommand{\arraystretch}{1.4}
    \begin{tabular}{@{}>{\raggedright\arraybackslash}p{3.2cm} 
                    >{\raggedright\arraybackslash}p{4cm} 
                    >{\raggedright\arraybackslash}p{5.2cm} 
                    >{\raggedright\arraybackslash}p{4cm}@{}}
        \toprule
        \textbf{$S$ (Initial Space)} & \textbf{$A$ (Selected Axioms)} & \textbf{$A'$ (Modified Axioms)} & \textbf{$S'$ (Transformed Space)} \\
        \midrule
        Newtonian Physics & Absolute time; separate space/time; gravity as force & Relative time; spacetime unity; gravity as curvature & Relativistic Physics\\
        Ptolemaic Astronomy & Earth-centered universe & Sun-centered model & Heliocentric system  \\
        Euclidean Geometry & Parallel Postulate: through a point not on a line there is exactly one parallel line & Altered Parallel Postulate: through such a point there are \emph{zero} (elliptic) or \emph{at least two} (hyperbolic) parallels & Non‑Euclidean geometry (elliptic and hyperbolic)  \\
        \bottomrule
    \end{tabular}

    \label{tab:transformational_creativity}
\end{table*}

%% file: sections/ai_superscientists.tex
In light of a series of recent works using AI for scientific ideation \cite{krenn_ideas,llms_novel,llm_sci_ideas,scimon}, experiment design \cite{curie}, and even end-to-end discovery \cite{zochi2025,sakana2024}, we believe our theory opens a promising avenue towards building the first AI super-scientists poised to make transformative breakthroughs. By coupling the wide-reaching creative abilities of large language models (LLMs) with our graphical theory above, we envision a neuro-symbolic approach to transformative scientific discovery that, given a particular scientific field or sub-field $\mathcal{D}$, can be built using the following components.
\begin{enumerate}
    \item{\textbf{Background Knowledge:}} A collection $\mathcal{C}$ of scientific literature (surveys, conference proceedings, journal papers, etc.) representing background knowledge about $\mathcal{D}$.
    \item{\textbf{Graph Construction:}} A DAG construction technique $M: \mathcal{C} \to \mathcal{G}$, where $\mathcal{G}$ is a space of graphs structured according to Definition~\ref{eq:concept_space}, representing the dependency structure of rules and axioms in $\mathcal{D}$.
    \item{\textbf{Problem Identification:}} The identification of a set $\mathcal{P}$ of ``open problems'' or anomalies that currently plague the research community in $\mathcal{D}$, obtained from the background literature $A : \mathcal{C} \to \mathcal{P}$.
    \item{\textbf{Idea Generation:}} Finally, an idea generation module $$I : \mathcal{L} \times \mathcal{G} \times \mathcal{P} \to \mathcal{T} \times \mathcal{G}$$ that maps from a space of large language models $\mathcal{L}$, dependency graphs $\mathcal{G}$, and open problems $\mathcal{P}$ to produce a transformative idea text $t \in \mathcal{T}$ and its corresponding graph transformation $G' \in \mathcal{G}$.
\end{enumerate}
This approach offers a natural way to incorporate knowledge graphs to enhance the use of LLMs for scientific discovery \cite{scimon,krenn_ideas}, a direction which has shown early promise in improving the technical detail of generated ideas \cite{scimon}. Moreover, by making graph transformations explicit, this technique also boasts significant interpretability benefits.

%% file: sections/conclusion.tex
In this paper, we provided a graphical theory of transformational scientific creativity, illustrated the transformative potential of modifying axioms, and used our framework to elucidate the transformativeness of several groundbreaking historical discoveries. 

Subsequent investigations could extend our theory by providing greater depth, breadth, and nuance to the framework, contextualized within the history and philosophy of science. For example, we suggest exploring in-depth historical illustrations for one to three exemplary discoveries, such as those in Table~\ref{tab:transformational_creativity}, showing how all components of the theory are instantiated and transformed. Similarly, it may be insightful to study a large collection of scientific discoveries, such as those addressed in \citeauthor{haven100} \shortcite{haven100} and \citeauthor{creative_comb_rep} \shortcite{creative_comb_rep}, and delineate a ``logic of discovery'' via common graph transformation patterns among these cases. Historical studies such as these may suggest criteria for ``good'' graph transformations and highlight distinctions between edge and vertex modifications.

One notable implication of our theory's rigid structure is that it may be best suited to scientific areas that involve formal systems, such as mathematics, engineering theory, and computer science. Nevertheless, in exchange for generality, we obtain more specificity regarding how the transformational creativity process can be formulated in these domains, opening the door for future theoretical or computational inquiry. Our work, like \citeauthor{santo2024formalcreativitytheorypreliminary} \shortcite{santo2024formalcreativitytheorypreliminary}, is a crucial first step towards formalizing transformational creativity.

%% file: sections/acknowledgements.tex
The authors thank Jonathan Livengood for helpful feedback on an earlier version of this paper.

%% file: iccc.bbl
\begin{thebibliography}{}

\bibitem[\protect\citeauthoryear{Aguiar, Atã, and Queiroz}{2015}]{aguiar2015semiotic}
Aguiar, D.; Atã, P.; and Queiroz, J.
\newblock 2015.
\newblock Intersemiotic translation and transformational creativity.
\newblock {\em Punctum. International Journal of Semiotics} 1:11--21.

\bibitem[\protect\citeauthoryear{Boden}{1992}]{boden92}
Boden, M.
\newblock 1992.
\newblock {\em The Creative Mind}.
\newblock London: Abacus.

\bibitem[\protect\citeauthoryear{Friedman \bgroup et al.\egroup }{2000}]{friedman2000directedcausallearning}
Friedman, N.; Linial, M.; Nachman, I.; and Pe'er, D.
\newblock 2000.
\newblock Using {B}ayesian networks to analyze expression data.
\newblock {\em Journal of Computational Biology} 7:601--620.

\bibitem[\protect\citeauthoryear{Gu and Krenn}{2024}]{krenn_ideas}
Gu, X., and Krenn, M.
\newblock 2024.
\newblock Interesting scientific idea generation using knowledge graphs and {LLMs}: Evaluations with 100 research group leaders.
\newblock arXiv:2405.17044.

\bibitem[\protect\citeauthoryear{Haven}{2007}]{haven100}
Haven, K.
\newblock 2007.
\newblock {\em 100 Greatest Science Discoveries of All Time}.
\newblock Bloomsbury Publishing USA.

\bibitem[\protect\citeauthoryear{Intology}{2025}]{zochi2025}
Intology.
\newblock 2025.
\newblock Zochi technical report.
\newblock https://www.intology.ai/blog/zochi-tech-report.

\bibitem[\protect\citeauthoryear{Jaccard}{1901}]{jaccard1901etude}
Jaccard, P.
\newblock 1901.
\newblock {\'E}tude comparative de la distribution florale dans une portion des {A}lpes et des {J}ura.
\newblock {\em Bull. Soc. Vaudoise Sci. Nat.} 37:547--579.

\bibitem[\protect\citeauthoryear{Koestler}{1964}]{koestler_creation}
Koestler, A.
\newblock 1964.
\newblock {\em The Act of Creation}.
\newblock Macmillan.

\bibitem[\protect\citeauthoryear{Kon \bgroup et al.\egroup }{2025}]{curie}
Kon, P. T.~J.; Liu, J.; Ding, Q.; Qiu, Y.; Yang, Z.; Huang, Y.; Srinivasa, J.; Lee, M.; Chowdhury, M.; and Chen, A.
\newblock 2025.
\newblock Curie: Toward rigorous and automated scientific experimentation with {AI} agents.
\newblock arXiv:2502.16069.

\bibitem[\protect\citeauthoryear{Kuhn}{1962}]{kuhn1997structure}
Kuhn, T.~S.
\newblock 1962.
\newblock {\em The Structure of Scientific Revolutions}.
\newblock University of Chicago Press.

\bibitem[\protect\citeauthoryear{Lu \bgroup et al.\egroup }{2024}]{sakana2024}
Lu, C.; Lu, C.; Lange, R.~T.; Foerster, J.; Clune, J.; and Ha, D.
\newblock 2024.
\newblock The {AI} scientist: Towards fully automated open-ended scientific discovery.
\newblock arXiv:2408.06292.

\bibitem[\protect\citeauthoryear{Ritchie}{2006}]{ritchie2006transformational}
Ritchie, G.
\newblock 2006.
\newblock The transformational creativity hypothesis.
\newblock {\em New Generation Computing} 24:241--266.

\bibitem[\protect\citeauthoryear{Rothenberg}{2014}]{flight_from_wonder}
Rothenberg, A.
\newblock 2014.
\newblock {\em Flight from Wonder: An Investigation of Scientific Creativity}.
\newblock Oxford University Press.

\bibitem[\protect\citeauthoryear{Santo, Wiggins, and Cardoso}{2024}]{santo2024formalcreativitytheorypreliminary}
Santo, L.~E.; Wiggins, G.; and Cardoso, A.
\newblock 2024.
\newblock Towards a formal creativity theory: Preliminary results in novelty and transformativeness.
\newblock arXiv 2405.02148.

\bibitem[\protect\citeauthoryear{Shi and Evans}{2023}]{evanssuprisingcombinations2023}
Shi, F., and Evans, J.
\newblock 2023.
\newblock Surprising combinations of research contents and contexts are related to impact and emerge with scientific outsiders from distant disciplines.
\newblock {\em Nature Communications} 14.

\bibitem[\protect\citeauthoryear{Shi, Foster, and Evans}{2015}]{shietalweavingfabricofscience}
Shi, F.; Foster, J.~G.; and Evans, J.~A.
\newblock 2015.
\newblock Weaving the fabric of science: Dynamic network models of science's unfolding structure.
\newblock {\em Social Networks} 43:73--85.

\bibitem[\protect\citeauthoryear{Si, Yang, and Hashimoto}{2024}]{llms_novel}
Si, C.; Yang, D.; and Hashimoto, T.
\newblock 2024.
\newblock Can {LLMs} generate novel research ideas? a large-scale human study with 100+ {NLP} researchers.
\newblock arXiv:2409.04109.

\bibitem[\protect\citeauthoryear{Simonton}{2004}]{creativity_in_science}
Simonton, D.~K.
\newblock 2004.
\newblock {\em Creativity in Science: Chance, Logic, Genius, and Zeitgeist}.
\newblock Cambridge University Press.

\bibitem[\protect\citeauthoryear{Simonton}{2021}]{simonton2021scientific}
Simonton, D.~K.
\newblock 2021.
\newblock Scientific creativity: Discovery and invention as combinatorial.
\newblock {\em Frontiers in Psychology} 12:721104.

\bibitem[\protect\citeauthoryear{Su \bgroup et al.\egroup }{2024}]{llm_sci_ideas}
Su, H.; Chen, R.; Tang, S.; Zheng, X.; Li, J.; Yin, Z.; Ouyang, W.; and Dong, N.
\newblock 2024.
\newblock Two heads are better than one: A multi-agent system has the potential to improve scientific idea generation.
\newblock arXiv:2410.09403.

\bibitem[\protect\citeauthoryear{Thagard}{2012}]{creative_comb_rep}
Thagard, P.
\newblock 2012.
\newblock Creative combination of representations: Scientific discovery and technological invention.
\newblock In {\em {Psychology of Science: Implicit and Explicit Processes}}. Oxford University Press.

\bibitem[\protect\citeauthoryear{Toivanen \bgroup et al.\egroup }{2018}]{toivanen2018novelsongs}
Toivanen, J.; Järvisalo, M.; Alm, O.; Ventura, D.; Vainio, M.; and Toivonen, H.
\newblock 2018.
\newblock Towards transformational creation of novel songs.
\newblock {\em Connection Science} 31:1--29.

\bibitem[\protect\citeauthoryear{Varshney}{2020}]{varshney_limit_theorems}
Varshney, L.~R.
\newblock 2020.
\newblock Limits theorems for creativity with intentionality.
\newblock In {\em Proceedings of the International Conference on Computational Creativity},  390--393.

\bibitem[\protect\citeauthoryear{Wang \bgroup et al.\egroup }{2024}]{scimon}
Wang, Q.; Downey, D.; Ji, H.; and Hope, T.
\newblock 2024.
\newblock {SciMON}: Scientific inspiration machines optimized for novelty.
\newblock In {\em Proceedings of the 62nd Annual Meeting of the Association for Computational Linguistics},  279--299.

\bibitem[\protect\citeauthoryear{Wiggins}{2006}]{wiggins2006}
Wiggins, G.
\newblock 2006.
\newblock A preliminary framework for description, analysis and comparison of creative systems.
\newblock {\em Knowledge-Based Systems} 19:449--458.

\end{thebibliography}
